\documentclass[a4paper]{article}
\usepackage[margin=1.2in]{geometry}
\usepackage{amsmath}
\usepackage{algorithmic}
\usepackage{algorithm}
\usepackage{amsthm}
\usepackage{amsfonts}
\usepackage{comment}
\usepackage{graphicx}
\usepackage{hyperref}
\usepackage{color}
\usepackage{algorithm,algorithmic}

\newtheorem{thm}{Theorem}

\newtheorem{lemma}{Lemma}

\newtheorem{remark}{Remark}

\usepackage[mathscr]{euscript}

\def \R {\mathbb{R}}
\def \x {\mathbf{x}}
\def \E {\mathrm{E}}

\def \L {\mathcal{L}}
\def \S {\mathcal{S}}

\def \F {\mathcal{F}}

\def \N {\mathcal{N}}

\def \y {\mathbf{y}}

\def \w {\mathbf{w}}
\def \wh {\widehat{\w}}

\def \v {\mathbf{v}}

\def \D {P}

\def \lb {\mathcal{L}}
\def \X {\mathcal{X}}
\def \rh {\widehat{\rho}}
\def \tt {\widetilde{t}}

\def \tr {\mbox{tr}}
\def \W {\mathscr{W}}
\def \le {\mathcal{L}}
\def \Y {\mathcal{Y}}
\def \N {\mathcal{N}}
\def \M {\mathbf{M}}
\def \Z {\mathbf{Z}}
\def \H {\mathbf{H}}
\def \I {\mathbf{I}}
\def \V {\mathbf{V}}
\def \tr {{\text{trace}}}
\title{Excess Risk Bounds for Exponentially Concave Losses}

\author{Mehrdad Mahdavi \\
\small{Michigan State University} \\ 
\small{\texttt{mahdavim@cse.msu.edu}}
\and Rong Jin\\  
\small{Michigan State University} \\
\small{\texttt{rongjin@cse.msu.edu}} 
}
\date{}
\begin{document} 
\maketitle

\begin{abstract}
The overarching goal of this paper is to derive excess risk bounds for learning from exp-concave loss functions in  passive and  sequential learning settings.  Exp-concave loss functions  encompass several fundamental problems in machine learning  such as squared loss in linear regression, logistic loss in classification, and negative logarithm loss in portfolio management.  In batch setting, we obtain sharp bounds on the performance of empirical risk minimization performed in a linear hypothesis  space and with respect to the exp-concave loss functions.  We also extend the results to the online setting  where the learner receives the training examples in a sequential manner. We propose an online learning algorithm that is a properly modified version of online Newton method to obtain sharp risk bounds.   Under an additional mild assumption on the loss function,  we show that in both settings we are able to achieve an excess risk bound of $O(d\log n/n)$ that holds with a \textit{high probability}.

\end{abstract}

%
\section{Introduction}
We investigate the excess risk bounds for learning a linear classifier using a exponentially concave (abbr. as exp-concave) loss function (see e.g.,~\cite{hazan-2007-logarithm} and~\cite{Cesa-Bianchi:2006:PLG}). More specifically, let $\mathcal{S} = \{(\x_1,y_1), (\x_2,y_2), \cdots, (\x_n,y_n)\} \in \Xi^n$  be a set of i.i.d.  training examples sampled from an unknown distribution $\D$ over instance space $\Xi = \X \times \Y$, where $\x_i \in \X \subseteq \R^d$ with $\|\x_i\| \leq 1$ and $\y_i \in \Y :=  \{-1, +1\}$ and $\y_i \in \Y :=  [-1, +1]$ in classification and regression problems, respectively. Let $\W = \{\w \in \R^d: \|\w\| \leq R\}$ be our domain of linear classifiers with bounded norm, where $R > 0$ determines the size of the domain. We  aim at finding a learner $\w \in \W$ with the assist of training samples $\S$ that generalizes well on unseen instances.

Let $\ell(z): \R \mapsto \R_{+}$ be the convex surrogate loss function used to measure the classification error.  In this work, we are interested  in learning problems where the loss function $\ell(z)$ is a one-dimensional exponentially concave  function with constant $\alpha > 0$ (i.e., $\exp(-\alpha z)$ is concave for any $|z| \leq R$). Examples of such loss functions are the squared loss  used in regression, logistic loss used in classification, and negative logarithm loss  used in portfolio management~\cite{korenopen,hazan-2007-logarithm,brenden-open,agarwal-2006-newton}. Similar to most analysis of generalization performance, we assume $\ell(z)$ to be Lipschitz continuous with constant $G$, i.e. $|\ell'(z)| \leq G$.   Define $\le(\w)$ as the expected loss function for an arbitrary classifier $\w \in \W$, i.e.
\[
\le(\w) = \E_{(\x, y) \sim \D}\left[ \ell(y\w^{\top}\x) \right].
\]
Let $\w_* \in \W$ be the optimal solution that minimizes $\le(\w)$ over the domain $\W$, i.e. $\w_* = \mathop{\arg\min}_{\w \in \W} \le(\w)$. We note that the exp-concavity of  individual loss functions $\ell(\cdot)$ also implies the exp-concavity of the expected function $\le(\w)$ (a straightforward proof can be found in~\cite[Lemma~1]{korenopen}). Our goal is to efficiently learn a classifier $\wh$ with the help of  training set $\mathcal{S}$  with small excess risk defined by:
$$ \mathscr{E}_P(\wh) := \le(\wh) - \min_{\w \in\W} \le(\w) = \le(\wh)  - \le(\w_*).$$

While the main focus of statistical learning theory was on understanding learnability and sample complexity by investigating the complexity of hypothesis class in terms of known combinatorial measures, recent advances in online learning and optimization theory opened a new trend in understanding  the generalization ability of learning  algorithms in terms of the characteristics of loss functions being used in convex learning problems. In particular, a staggering number of results have focused on strong convexity of loss function (that is a stronger condition than exp-concavity) and obtained better generalization bounds which are referred to  as \textit{fast} rates~\cite{kakade2008generalization,sridharan2008fast}.  In terms of smoothness of loss function, a recent result~\cite{srebro2010smoothness} has shown that under smoothness assumption, it is possible to obtain \textit{optimistic} rates (in the sense that smooth losses yield better generalization bounds when the problem is easier and the expected loss of optimal classifier is small), which are more appealing than Lipschitz continuous cases.  This work extends the results to  exp-concave loss functions and investigates how to  obtain sharper excess risk bounds for learning  from such functions. We note that although the online Newton method~\cite{hazan-2007-logarithm} yields $O(d\log n)$ regret bound, it is only able to achieve an $O(d\log n/n)$ bound for excess risk in expectation. In contrast, the excess risk bounds analyzed in this work are all in high probability sense.

We consider two settings to learn a classifier from the provided training set $\S$.  In \textit{statistical setting} (also called batch learning)~\cite{bousquet2004introduction}, we assume that the learner has access to all training examples in advance, and in \textit{online setting} the examples are assumed to become available to the learner one at a time. We show that with an additional assumption regarding the exponential concave loss function, we will be able to achieve an excess risk bound of $O(d\log n/n)$, which is significantly faster than $O(1/\sqrt{n})$ rate  for general convex Lipschitz loss functions.  The proof of batch setting utilizes the notion of local Radamacher complexities and involves novel ingredients tailored to exp-concave functions in order to obtain sharp convergence rates.  In online setting, the results  follows from Bernstein inequality for martingales and peeling process.  We note that fast rates are possible and well known  in sequential prediction via the notion of mixable losses~\cite{vovk1995game}, and in batch setting under Tsybakov's margin condition with $\kappa = 1$~\cite{tsybakov2004optimal}, where  the relation between these two settings has been recently investigated  via the notion of stochastic mixability~\cite{van2012mixability}.  However,  our analysis and conditions are different and only focuses on the exp-concavity property of  the loss  to derive an $O(\log n/n)$ risk bound.


%
\section{The Algorithms}

We study two algorithms for learning with exp-concave loss functions. The first algorithm that is  devised for  batch setting  is simply based on empirical risk minimization. More specifically, it learns a classifier from the space of linear classifiers $\W$ by solving the following   optimization problem
\begin{eqnarray}
\min\limits_{\w \in \W} \frac{1}{n}\sum_{i=1}^n \ell(y_i\w^{\top}\x_i). \label{eqn:opt}
\end{eqnarray}
The optimal solution to (\ref{eqn:opt}) is denoted by $\wh_*$. Here, we are not concerned with the optimization procedure to find $\wh_*$ and only investigate the access risk of obtained classifier $\wh_*$ with respect to the optimal classifier $\w_*$.

\begin{algorithm}[t]
\caption{\texttt{Stochastic  Exp-concave Optimization}}
\begin{algorithmic}[1]

\STATE {\bf Input:} step size $\eta_1 >0$ and smoothing parameter $a > 0$

\STATE {\bf Initialization:} $\w_1 = \mathbf{0}$ and $\M_0 = a \I$

\FOR{$i = 1, \ldots, n$}
    \STATE Receive training example $(\x_i, y_i)$
    \STATE Compute $\M_i = \M_{i-1} + \x_i\x_i^{\top}$ and the covariance matrix $\Z_i = \M_i/i$
    \STATE Compute the gradient $\v_i = \ell'(y_i\w_i^{\top}\x_i)\x_i$ and step size $\eta_i = \eta_1/i$
    \STATE Update the solution $\w_i$ by solving the following optimization problem
    \begin{eqnarray}
        \w_{i+1} = \mathop{\arg\min}\limits_{\w \in \W} \eta_i\langle \w, \v_i \rangle + \frac{1}{2}\|\w - \w_i\|_{\Z_i}^2 \label{eqn:update-1}
    \end{eqnarray}
    where $\|\w\|_{\Z_i} = \w^{\top} \Z_i \w$.
\ENDFOR
\RETURN $\wh_* = \frac{1}{n}\sum_{i=1}^n \w_i$.
\end{algorithmic} \label{alg:2}
\end{algorithm}

Our second algorithm is a modified online Newton method~\cite{hazan-2007-logarithm}. Algorithm~\ref{alg:2} gives the detailed steps. The key difference between Algorithm~\ref{alg:2} and the online Newton algorithm~\cite{hazan-2007-logarithm} is that at each iteration, it estimates a smoothed version of the covariance matrix $\Z$ using the training examples received in the past. In contrast, the online Newton method takes into account the gradient $\ell'(y_i\w_i^{\top}\x_i)$ when updating $\Z_i$. It is this difference that allows us to derive an $O(d\log n /n)$ excess risk bound for the learned classifier.  The classifier learned from the online algorithm $\wh_*$  is simply the average of solutions obtained over all iterations. We also note that the idea of using an estimated covariance matrix for online learning and optimization has been examined by several studies~\cite{crammer2010learning,crammer2009adaptive,duchi2011adaptive,orabona2010new}. It is also closely related to the technique of time varying potential discussed in~\cite{cesa2002second,Cesa-Bianchi:2006:PLG} for regression. Unlike these studies  that are mostly focused on obtaining regret bound, we aim to  study the excess risk bound for the learned classifier.

%
\section{Main Results}

We first state the result of batch learning problem  in (\ref{eqn:opt}), and then the result of online learning algorithm that is detailed in Algorithm~\ref{alg:2}. In order to achieve an excess risk bound better than $O(1/\sqrt{n})$, we introduce following key assumption for the analysis of the empirical error minimization  problem in (\ref{eqn:opt})
\begin{eqnarray*}
& \mbox{\bf Assumption I} & \mbox{there exists a constant $\theta > 0$ s. t. } \E\left[[\ell'(y\w_*^{\top}\x)]^2\x\x^{\top} \right] \succeq \theta \E\left[\x\x^{\top}\right].
\end{eqnarray*}
For the online learning method in~Algorithm~\ref{alg:2}, we strengthen Assumption (I) as:
\begin{eqnarray*}
& \mbox{\bf Assumption II} & \mbox{there exists a constant $\theta > 0$ s. t. } \E\left[[\ell'(y\w^{\top}\x)]^2\x\x^{\top} \right] \succeq \theta \E[\x\x^{\top}], \forall \w \in \W.
\end{eqnarray*}
Note that unlike Assumption (I) that only requires the property to hold with respect to  the optimal solution $\w_*$, Assumption (II) requires the property to hold for any $\w \in \W$, making a stronger assumption than Assumption (I). We also note that Assumption (II) is closely related to strong convexity assumption. In particular, it is easy to verify that when $\E[\x\x^{\top}]$ is strictly positive definite, the expected loss $\L(\w)$ will be strongly convex in $\w$ by using the property of exponential concave function. 

The following lemma shows a general scenario when both Assumptions (I) and (II) hold.
\begin{lemma} \label{lem:0}
Suppose (i) $\Pr(y=1|\x) \geq q$ and $\Pr(y=-1|\x) \geq q$ for any $\x \in \mathcal{X}$, where $q > 0$, and (ii) $\ell'(0) >0$. Then Assumption (I) and (II) hold with $\theta \geq q[\ell'(0)]^2$
\end{lemma}
\begin{proof}
We first bound $\E\left[[\ell'(y\w^{\top}\x)]^2|\x\right]$ for any given $\x \in \X$ by
\[
\E\left[[\ell'(y\w^{\top}\x)]^2|\x\right] \geq q\left([\ell'(\w^{\top}\x)]^2 + [\ell'(-\w^{\top}\x)]^2\right)
\]
Since $\ell'(z)$ is monotonically increasing function, we have
\[
|\ell'(0)| \leq \max(|\ell'(\w^{\top}\x)|, |\ell'(-\w^{\top}\x)|)
\]
and therefore
\[
\E\left[[\ell'(y\w^{\top}\x)]^2|\x\right] \geq q[\ell'(0)]^2
\]
implying that $\E\left[[\ell'(y\w^{\top}\x)]^2 \x\x^{\top} \right] \geq q[\ell'(0)]^2\E[\x\x^{\top}]$ as desired.
\end{proof}
We note that $\ell'(0) > 0$ is the necessary and sufficient condition that the convex surrogate loss function for 0-1 loss function to be classification-calibrated~\cite{bartlett-2003-convexity}, and therefore is almost unavoidable if our final goal is to minimize the binary classification error.

The excess risk bound for the batch learning algorithm is given in the following theorem.
\begin{thm} \label{thm:excess-risk}
Suppose Assumption (I) holds. Let $\wh_*$ be the solution to the convex optimization problem in~(\ref{eqn:opt}). Define
\[
\gamma = \max\left(1, \frac{G^2}{\theta}, \frac{G}{\alpha \theta R}\right), \quad \rho_0 = \max\left(32\gamma, \sqrt{\gamma\left(28 + \frac{3}{GR}\right)}\right).
\]
Then with a probability $1 - 2me^{-t}$, where $m = \lceil \log_2 n\rceil$, we have
\[
\lb(\wh_*) - \lb(\w_*) \leq \left(GR\left[32 \rho_0 + 28\right] + 3\right)\frac{t + 2 + d\log n}{n} = \tilde{O}\left( \frac{d\log n}{n} \right).
\]
\end{thm}

The following theorem provides the excess risk bound for Algorithm~\ref{alg:2} where training examples ae received in an online fashion and the final solution is reported as the average of all the intermediate solution.

\begin{thm} \label{thm:1}
Suppose Assumption (II) holds. Let $\wh_*$ be the average solution returned by Algorithm~\ref{alg:2}, with $\eta_1 = \max(1, 3/[\theta\beta])$ and $a= \eta_1^2G^2d/[4R^2]$. With a probability $1 - 2me^{-t}$, where $m = \lceil \log_2 n \rceil$, we have
\[
\lb(\wh_*) - \lb(\w_*) \leq \rho_0GR\frac{d}{n}\log\left( 1 + \frac{4 n}{\gamma^2 d^2}\right) + 3GR(2\gamma + 1)t = \tilde{O}\left( \frac{d\log n}{n} \right).
\]
\end{thm}

\begin{remark}
As indicated in Theorems~\ref{thm:excess-risk} and \ref{thm:1}, the excess risk for both batch learning and online learning is reduced at the rate of $O(d\log n/n)$, which is consistent with the regret bound for online optimizing the exponentially concave loss functions~\cite{hazan-2007-logarithm}. We note that the linear dependence on $d$ is in general unavoidable. This is because when $\E[\x\x^{\top}]$ is strictly positive definite, the function $\L(\w)$ will be strongly convex with modulus proportion to $\lambda_{\min}\left(\E[\x\x^{\top}]\right)$. Since $\lambda_{\min}\left(\E[\x\x^{\top}]\right) \propto 1/d$, we would expect a linear dependence on $d$ based on the minimax convergence rate of stochastic optimization for strongly convex function. Finally, we note that for strongly convex loss functions, it is known that an $O(1/n)$ excess risk bound can be achieved without the $\log n$ factor. It is however unclear if the $\log n$ factor can be removed from the excess risk bounds for exponential concave functions, a question to be investigated in the future. 

Comparing the result of online learning with that of batch learning, we observe that, although both achieve similar excess risk bounds, batch learning algorithm is advantageous in two aspects. First, the batch learning algorithm has to make a weaker assumption about the data (i.e. Assumption (I) vs. Assumption (II)). Second, the batch learning algorithm does not have to know the parameter $\alpha$ and $\theta$ in advance, which is important for online learning method to determine the step size $\eta_1$. 
\end{remark}
%
\section{Analysis}
We now turn to the proofs of our main results. The main steps in each proof are provided in the
main text, with some of the more technical results deferred to the appendix.

\subsection{Proof of Theorem~\ref{thm:excess-risk}}
Our analysis for batch setting is based on the  Talagrand's inequality and in particular its variant (Klein-Rio bound) with improved constants  derived in~\cite{klein2005concentration} (see also~\cite[Chapter~2]{vladimir-2011-oracle}). To do so, we define
\[
\|P_n - P\|_{\W} = \sup\limits_{\w \in \W} \left|\frac{1}{n}\sum_{i=1}^n \left [ \ell(y_i\w^{\top}\x_i) - \ell(y_i\w_*^{\top}\x_i)\right] - \E_{(\x, y)}\left[\ell(y\w^{\top}\x) - \ell(y\w_*^{\top}\x)\right] \right|
\]
and
\[
U(\W) = \max\limits_{\w \in \W, \|\x\| \leq 1} \ell(y\w^{\top}\x) - \ell(y\w_*^{\top}\x), \; \sigma_P(\W) = \sup\limits_{\w \in \W} \E_{(\x, y)}\left[(\ell(y\w^{\top}\x) - \ell(y\w_*^{\top}\x))^2 \right].
\]
The analysis is rooted in the following concentration inequality:
\begin{thm} \label{thm:Talagrand}
We have
\[
\Pr\left\{\|P_n - P\|_{\W} \geq 2\E\|P_n - P\|_{\W} + \sigma_P(\W)\sqrt{\frac{2t}{n}} + \frac{(U(\W) +3)t}{3n} \right\} \leq e^{-t}.
\]
\end{thm}
The following property of exponential concave loss function from~\cite{hazan-2007-logarithm} will be used throughout the paper.
\begin{thm} \label{thm:exponential-concave}
If a function $f:\W \mapsto \R$ is such that $exp(-\alpha f(\w))$ is concave, and
has gradient bounded by $\|\nabla f\| \leq G$, then there exists $\beta \leq \frac{1}{2}\min(\alpha, 1/[4GR])$ such that the following holds
\[
f(\w) \geq f(\w') + (\w - \w')^{\top}\nabla f(\w') + \frac{\beta}{2}\left[  \nabla f(\w')^{\top} (\w - \w') \right]^2, \forall \w, \w' \in \W.
\]
\end{thm}

The key quantity for our analysis is the following random variable:
\[
\rho(\w) = \frac{1}{2R}\sqrt{\E[|\x^{\top}(\w - \w_*)|]^2}.
\]
Evidently, $\rho(\w) \leq 1, \forall \w \in \W$. The following lemma deals with the concentration of $\rho(\w)$.
\begin{lemma}~\label{lem:1}
Define $\Delta = \left\{\w \in \W: \rho(\w) \leq \rho \right\}$. Then, with a probability $1 - e^{-t}$, we have,
\[
\frac{1}{n}\sup\limits_{\w \in \Delta} \sum_{i=1}^n \left[\x_i^{\top}(\w - \w_*)\right]^2 \leq 10R^2\left(\rho^2 + \frac{t + 1 + d\log n}{n} \right)
\]
\end{lemma}
\begin{proof}
Fix a $\w \in \Delta :=  \left\{\w \in \W: \rho(\w) \leq \rho \right\}$. Using the standard Bernstein's inequality~\cite{boucheron2004concentration}, we have, with a probability $1 - e^{-t}$,
\[
\left|\frac{1}{n}\sum_{i=1}^n [\x_i^{\top}(\w - \w_*)]^2 - \E[((\w - \w_*)^{\top}\x)^2]\right| \leq \frac{16R^2 t}{3n} + 2R\sqrt{\frac{2\E[((\w - \w_*)^{\top}\x)^2]t}{n}}
\]
By   definition of  the  domain $\Delta$, i.e., $\rho(\w) = \frac{1}{2R}\sqrt{\E[|\x^{\top}(\w - \w_*)|]^2} \leq \rho$ and above concentration result we obtain:
\[
\frac{1}{n}\sum_{i=1}^n [\x_i^{\top}(\w - \w_*)]^2 \leq 4R^2\left(\rho^2 + \frac{4t}{3n} + \rho\sqrt{\frac{2 t}{n}}\right) \leq 10R^2\left(\rho^2 + \frac{t}{n}\right)
\]
Next, we consider a discrete version of the space $\Delta$. Let $\N(\Delta, \epsilon)$ be the proper $\epsilon$-net of $\Delta$. Since $\Delta \subseteq \W$, we have
\[
|\N(\Delta, \epsilon)| \leq |\N(\W, \epsilon)| \leq \left(\frac{3R}{\epsilon}\right)^d
\]
Using the union bound, we have, with a probability $1 - e^{-t}$, for any $\w \in \N(\Delta, 3R/\sqrt{n})$,
\[
\frac{1}{n}\sum_{i=1}^n [\x_i^{\top}(\w - \w_*)]^2 \leq 10R^2\left(\rho^2 + \frac{t + d \log n}{n}\right)
\]
Since for any $\w \in \Delta$, there exists $\w' \in \N(\Delta, 3R/\sqrt{n})$, such that $\|\w - \w'\|_2 \leq 3R/\sqrt{n}$, we have, with a probability $1 - e^{-t}$,  for any $\w \in \Delta$,
\[
\frac{1}{n}\sum_{i=1}^n [\x_i^{\top}(\w - \w_*)]^2 \leq 10R^2\left(\rho^2 + \frac{t + 1 + d \log n}{n} \right),
\]
as desired.

\end{proof}
Define $\rh = \rho(\wh_*)$. The next theorem allows us to bound the excess risk using the random variable $\rh$.
\begin{thm} \label{thm:rho}
With a probability $1 - 2me^{-t}$, where $m = \lceil \log n \rceil$, we have
\[
\lb(\wh_*) - \lb(\w_*) \leq GR\left(26\left[\rh \sqrt{\frac{\tt}{n}} + \frac{\tt}{n} \right] + 6\rh\sqrt{\frac{\tt}{n}} + \frac{2\tt}{n}\right) + \frac{3 \tt}{n}
\]
where $\tt = t + 2 + d\log n$.
\end{thm}
Taking this statement as given for the moment, we proceed with the proof of Theorem~\ref{thm:excess-risk}, returning later to establish the claim stated in Theorem~\ref{thm:rho}. Our overall strategy of proving Theorem~\ref{thm:excess-risk} is to first bound $\rh$ by using the property of exp-concave function and the result from Theorem~\ref{thm:rho}, and then bound the excess risk. More specifically, using the result from Theorem~\ref{thm:rho}, we have, with a probability at least $1 - 2me^{-t}$,
\begin{eqnarray}
\lb(\wh_*) - \lb(\w_*) \leq GR\left(26\left[\rh \sqrt{\frac{\tt}{n}} + \frac{\tt}{n} \right] + 6\rh\sqrt{\frac{\tt}{n}} + \frac{2\tt}{n}\right) + \frac{3 \tt}{n} \label{eqn:bound-2-1}
\end{eqnarray}
Using the property of exp-concave  loss functions stated in Theorem~\ref{thm:exponential-concave}, we have
\begin{eqnarray*}
\lb(\wh_*) - \lb(\w_*) & \geq & (\wh_* - \w_*)^{\top}\nabla \lb(\w_*) + \frac{\beta}{2}\E\left[[\ell'(y\w_*^{\top}\x)(\wh_* - \w_*)^{\top}\x]^2 \right] \\
& \geq & \frac{\beta}{2}\E\left[[\ell'(y\w_*^{\top}\x)(\wh_* - \w_*)^{\top}\x]^2 \right]
\end{eqnarray*}
where the second step follows from the fact that $\w_*$ minimizes $\lb(\w)$ over the domain $\W$ and as a result $(\wh_* - \w_*)^{\top}\nabla \lb(\w_*) \geq 0$. We then use Assumption (I) to get
\[
\E\left[[\ell'(y\w_*^{\top}\x)(\wh_* - \w_*)^{\top}\x]^2 \right] \geq \theta \E\left[((\wh_* - \w_*)^{\top}\x)^2\right] = 4\theta R^2\rh^2
\]
and therefore
\begin{eqnarray}
\lb(\wh_*) - \lb(\w_*) \geq 2\beta\theta R^2 \rh^2. \label{eqn:bound-2-2}
\end{eqnarray}
Combining the bounds in (\ref{eqn:bound-2-1}) and (\ref{eqn:bound-2-2}), we have, with a probability $1 - 2me^{-t}$,
\[
\rh^2 \leq \frac{G}{2\beta \theta R}\left(32\rh\sqrt{\frac{\tt}{n}} + 28\frac{\tt}{n}\right) + \frac{3\tt}{2\beta\theta R^2 n}
\]
implying that
\[
\rh \leq \max\left(\frac{32 G}{\beta\theta R} , \sqrt{\frac{G}{\beta\theta R}\left(28 + \frac{3}{GR}\right)} \right)\sqrt{\frac{\tt}{n}}.
\]
We derive the final bound for $\rh$ by plugging the bound for $\beta$. The excess risk bound is completed by plugging the above bound for $\rh$.

We now turn to proving the result stated in Theorem~\ref{thm:rho}.\\

\begin{proof}[of Theorem~\ref{thm:rho}]
Our analysis will be based on the technique of local Rademacher complexity~\cite{bartlett2005local,koltchinskii2006local,vladimir-2011-oracle}.  The notion of local Rademacher complexity works by considering Rademacher averages of smaller subsets
of the hypothesis set. It generally leads to sharper learning bounds which, under certain general conditions, guarantee a faster convergence rate. Define $\rho_0 = 1/n$. We divide the range $[\rho_0, 1]$ into $m = \lceil\log_2 n\rceil$ segments, with $[\rho_0, \rho_1]$, $[\rho_1, \rho_2]$, ..., $[\rho_{m-1}, \rho_{m}]$, where $\rho_k = \rho_0 2^k$. Let $\rh = \rho(\wh_*)$. Note that $\rh$ is a random variable depending on the sampled training examples.

As the first step, we assume that $\rh \in [\rho_k, \rho_{k+1}]$ for some fixed $k$. Define domain $\Delta$ as
\[
\Delta = \left\{ \w \in \W: \rho(\w) \leq \rho_{k+1} \right\}
\]
Using the Telegrand inequality, with a probability at least $1 - e^{-t}$, we have
\begin{eqnarray}
\lb(\wh_*) - \lb(\w_*) \leq 2\E\|P_n - P\|_{\Delta} + \sigma_P(\Delta)\sqrt{\frac{2t}{n}} + \frac{\left(U(\Delta) + 3\right)t}{n} \label{eqn:bound-1-1}
\end{eqnarray}
We now bound each item on the right hand side of (\ref{eqn:bound-1-1}). First, we bound $\E\|P_n - P\|_{\Delta}$ as
\begin{eqnarray*}
\E\|P_n - P\|_{\Delta} & = & \frac{2}{n}\E\left[\sup\limits_{\w \in \Delta} \sum_{i=1}^n \sigma_i\left(\ell(y_i\w^{\top}\x_i) - \ell(y_i\w_*^{\top}\x_i)\right) \right] \\
& \leq & \frac{4G}{n}\E\left[\sup\limits_{\w \in \Delta} \sum_{i=1}^n y_i\sigma_i\x_i^{\top}(\w - \w_*)\right]
\end{eqnarray*}
where $\sigma_1, \ldots, \sigma_n$ are Rademacher random variables and the second step utilizes the contraction property of Rademacher complexity.

To bound $\E\|P_n - P\|_{\Delta}$, we need to  bound $\sup_{\w\in \Delta} \sum_{i=1}^n [\x_i^{\top}(\w - \w_*)]^2$. Using Lemma~\ref{lem:1}, we have, with a probability $1 - e^{-t}$,
\begin{eqnarray*}
\E\|P_n - P\|_{\Delta} & \leq & \frac{4G}{\sqrt{n}}\sqrt{\sup\limits_{\w \in \Delta} \sum_{i=1}^n [\x_i^{\top}(\w - \w_*)]^2} \leq \frac{13GR}{\sqrt{n}}\sqrt{\rho_{k+1}^2 + \frac{t + 1 + d\log n}{n}} \\
& \leq & 13GR\left(\frac{\sqrt{t + 1 + d\log n}}{n} + \frac{\rho_{k+1}}{\sqrt{n}} \right)
\end{eqnarray*}
Next, we bound $\sigma_P(\Delta)$ and $U(\Delta)$, i.e.
\begin{eqnarray*}
\lefteqn{\sigma^2_P(\Delta)} \\
& \leq& \sup_{\w \in \Delta} \E\left[(\ell(y\w^{\top}\x) - \ell(y\w_*^{\top}\x))^2 \right] \\
& \leq&  \sup_{\w \in \Delta}G^2\E[((\w - \w_*)^{\top}\x)^2] =  4R^2G^2\rho_{k+1}^2
\end{eqnarray*}
and $U(\Delta) \leq 2GR$. By putting the above results together, under the assumption $\rh \in [\rho_k, \rho_{k+1}]$, we have, with a probability $1 - 2 e^{-t}$,
\begin{eqnarray}
\lb(\wh_*) - \lb(\w_*) \leq 26GR\left[\frac{\rho_{k+1}}{\sqrt{n}} + \frac{\sqrt{t+1+d\log n}}{n} \right] + 2GR\rho_{k+1}\sqrt{\frac{2t}{n}} + \frac{(2GR + 3)t}{n} \label{eqn:bound-1-2}
\end{eqnarray}
Define $\tt = t + 2 + d\log n$. Using the fact that $\rho_{k+1} \leq 2\rh$, we can rewrite the bound in (\ref{eqn:bound-1-2}) as
\[
\lb(\wh_*) - \lb(\w_*) \leq 26GR\left[\frac{\rh \tt}{\sqrt{n}} + \frac{\sqrt{\tt}}{n} \right] + 6GR\rh\sqrt{\frac{\tt}{n}} + \frac{(2GR + 3)\tt}{n}
\]
By taking the union bound over all the segments, with probability  $1 - 2me^{-t}$, for any $\rh \in [\rho_0, 1]$, we have
\begin{eqnarray}
\lb(\wh_*) - \lb(\w_*) \leq 26GR\left[\rh\sqrt{\frac{\tt}{n}} + \frac{\sqrt{\tt}}{n} \right] + 6GR\rh\sqrt{\frac{\tt}{n}} + \frac{(2GR + 3)\tt}{n} \label{eqn:bound-1-3}
\end{eqnarray}
Finally, when $\rh \leq \rho_0 = 1/n$, we obtain
\begin{eqnarray}
\lb(\wh_*) - \lb(\w_*) \leq 2R\rho_0G \leq \frac{2RG}{n} \label{eqn:bound-1-4}
\end{eqnarray}
We complete the proof by combining the bounds in (\ref{eqn:bound-1-3}) and (\ref{eqn:bound-1-4}).
\end{proof}
%
\subsection{Proof of Theorem~\ref{thm:1}} \label{sec:stochastic}
We now turn to proving the main result on the excess risk for online setting. Define the covariance matrix $\H$ as $\H = \E_{\x}[\x\x^{\top}]$.  The following theorem bounds $\lb(\w_i) - \lb(\w_*)$ by exploiting the property of exponentially concave functions (i.e., Theorem~\ref{thm:exponential-concave}) and Assumption (II). Define $\delta_i$ as
\begin{eqnarray}
    \delta_i =  \nabla \lb(\w_i) - \ell'(y_i \x_i^{\top}\w_i) \x_i. \label{eqn:delta}
\end{eqnarray}
\begin{lemma} \label{lem:bound-1}
Suppose Assumption (II) holds. We have
\begin{eqnarray}
\lefteqn{\lb(\w_i) - \lb(\w_*) + \frac{\theta \beta}{3}\|\w_i - \w_*\|_\H^2} \nonumber \\
& \leq & \frac{\|\w_i - \w_*\|_{\M_{i-1}}^2}{2\eta_1} - \frac{\|\w_{i+1} - \w_*\|_{\M_{i}}^2}{2\eta_1} + \frac{\eta_1 G^2}{2}\x_i^{\top}\M^{-1}_i\x_i + (\w_i - \w_*)^{\top}\delta_i \nonumber \\
&      & + \frac{\theta\beta}{6}\left( \left[(\w_i - \w_*)^{\top}\x_i\right]^2- \|\w_i - \w_*\|_\H^2\right), \label{eqn:bound-iter}
\end{eqnarray}
where $\|\w\|_{\H}^2 = \langle \w, \H\w \rangle$.
\end{lemma}
\begin{lemma} We have \label{lemma:trace}
\[
   \x_i^{\top}\M_i^{-1}\x_i \leq \ln \det(\M_i) - \ln \det(\M_{i-1})
\]
\end{lemma}
By using Lemma~\ref{lemma:trace} and adding the inequalities in (\ref{eqn:bound-iter}) over all the iterations, we have
\begin{eqnarray}
\lefteqn{\sum_{i=1}^{n}\lb(\w_i) - \lb(\w_*) + \frac{\theta\beta}{3} \sum_{i=1}^{n}\|\w_i - \w_*\|^2_{\H}} \nonumber \\
& \leq & \frac{\|\w_1 - \w_*\|_{\M_0}^2}{2\eta_1} + \frac{\eta_1 G^2}{2}\underbrace{\left(\log\det(\M_n) - \log \det(\M_0)\right)}_{\equiv \Delta_1} + \underbrace{\sum_{i=1}^{n} (\w_i - \w_*)^{\top}\delta_i}_{\equiv \Delta_2} \nonumber \\
&  & + \frac{\theta\beta}{6}\underbrace{\sum_{i=1}^n \left[(\w_i - \w_*)^{\top}\x_i\right]^2- \|\w_i - \w_*\|_\H^2}_{\equiv \Delta_3}. \label{eqn:comb-bound}
\end{eqnarray}
We will bound $\Delta_1$, $\Delta_2$, and $\Delta_3$, separately. We start by bounding $\Delta_1$ as indicated by the following lemma.
\begin{lemma} \label{lem:logdet}
\[
\Delta_1 \leq d\log\left(1 + \frac{n}{ad}\right)
\]
\end{lemma}
To bound $\Delta_2$, we define  $A  =  \sum_{i=1}^n \|\w_i - \w_*\|_{\H}^2$. Using the Berstein inequality for martingale~\cite{boucheron2004concentration} and peeling process~\cite{vladimir-2011-oracle}, we have the following lemmas for bounding $\Delta_2$ and $\Delta_3$.
\begin{lemma} We have\label{lem:Delta2}
\[
\Pr\left(A \leq \frac{4R^2}{n}\right) + \Pr\left(\Delta_2 \leq \left[\frac{6G^2}{\theta\beta} + G R\right] t + \frac{\theta\beta}{6}A\right) \geq 1 - me^{-t}
\]
where $ m = \lceil 2 \log_2 n \rceil$.
\end{lemma}

\begin{lemma} We have\label{lem:Delta3}
\[
\Pr\left(A \leq \frac{4R^2}{n}\right) + \Pr\left(\Delta_3 \leq 8R^2 t + A\right) \geq 1 - me^{-t}
\]
where $m = \lceil 2 \log_2 n \rceil$.
\end{lemma}

First, we consider the case when $A \leq 4R^2/n$ and show the  following bound.
\begin{lemma} \label{lemma:bound-3}
Assume that the condition $A \leq 4R^2/n$ holds. We have
\begin{eqnarray}
\sum_{i=1}^{n} \lb(\w_i) - \lb(\w_*) + \frac{\theta\beta}{2}\sum_{i=1}^{n} \|\w_i - \w_*\|_{\H}^2 \leq 2RG. \label{eqn:bound-15}
\end{eqnarray}
\end{lemma}
Second, we assume that the following two conditions hold
\[
\Delta_2 \leq \left[\frac{6G^2}{\theta\beta} + G R\right] t + \frac{\theta\beta}{6}A, \; \Delta_3 \leq 8R^2 t + A
\]
Combining the above conditions with Lemma~\ref{lem:logdet} and using the inequality in (\ref{eqn:comb-bound}), we have
\[
\sum_{i=1}^{n} \lb(\w_i) - \lb(\w_*) \leq \frac{2aR^2}{\eta_1} + \frac{\eta_1 G^2}{2}d\log\left(1 + \frac{n}{ad}\right) + \left[\frac{6G^2}{\theta\beta} + G R + \frac{4\theta\beta}{3}R^2\right] t.
\]
Using the fact that $\eta_1 \geq 3/[\theta\beta]$, we set $a = \eta_1^2 G^2 d/[4R^2]$, we have
\[
\sum_{i=1}^{n} \lb(\w_i) - \lb(\w_*) \leq \frac{3G^2}{\theta\beta} d\log\left(1 + \frac{4 R^2 \theta \beta^2 n}{G^2 d^2}\right) + \left[\frac{6G^2}{\theta\beta} + G R + \frac{4\theta\beta}{3}R^2\right] t.
\]
We complete the proof by combining the two cases.

%
\section{Conclusions and Future Work} \label{sec:conclusion}

In this work, we addressed the generalization ability of learning from exp-concave loss functions in batch and online settings. For both cases we show that the excess risk bound can be bounded by $O(d\log n/n)$ when the learning is performed in a linear hypothesis space with dimension $d$ and with the help of $n$ training examples.

One open question to be addressed in the future is if $\log n$ factor can be removed from the excess risk bound for exponentially concave loss functions by a more careful analysis. Another open question that needs to be investigated in the future is to improve the dependence on $d$ if we are after a sparse solution. According to the literature of sparse recovery~\cite{koltchinskii2011oracle} and optimization~\cite{AgarwalNW12}, we should be able to replace $d$ with $s\log d$ in the excess risk bound if we restrict the optimal solution to a sparse one. In the future, we plan to explore the technique of sparse recovery in analyzing the generalization performance of exponential concave function to reduce the dependence on $d$.

\bibliographystyle{ieeetr}
\bibliography{exponential-concave}

\appendix
\section*{Appendix A. Proof of Lemma~\ref{lem:bound-1}}
From the exp-concavity of expected loss function we have
\[
\lb(\w_*) \geq \lb(\w_i) + (\w_* - \w_i)^{\top}\nabla \lb(\w_i) + \frac{\beta}{2}(\w_i - \w_*)^{\top}\E\left[[\ell'(y\w_i^{\top}\x)]^2\x\x^{\top} \right] (\w_i - \w_*).
\]
Combining the above inequality with our assumption
\[
\E\left[[\ell'(y\w_i^{\top}\x)]^2\x\x^{\top} \right] \succeq \theta \E[\x\x^{\top}],
\]
and rearranging the terms results in the following inequality
\[
\lb(\w_i) - \lb(\w_*) + \frac{\theta \beta}{2}\|\w_i - \w_*\|_\H^2 \leq (\w_i - \w_*)^{\top}\nabla \lb(\w_i).
\]
Applying  the fact that
\begin{eqnarray*}
& & (\w_i - \w_*)^{\top}\nabla \lb(\w_i) \\
&=& \ell'(y\w_i^{\top}\x)(\w_i - \w_*)^{\top}\x + (\w_i - \w_*)\left(\nabla \lb(\w_i) - \ell'(y\w_i^{\top}\x)\x\right) \\
&=& \frac{\|\w_i - \w_*\|_{\Z_i}^2}{2\eta_i} - \frac{\|\w_{i+1} - \w_*\|_{\Z^{-1}_i}^2}{2\eta_i} + \frac{\eta G^2}{2}\x_i^{\top}\Z_i\x_i + (\w_i - \w_*)^{\top}\delta_i
\end{eqnarray*}
we obtain
\begin{eqnarray*}
\lefteqn{\lb(\w_i) - \lb(\w_*) + \frac{\theta \beta}{2}\|\w_i - \w_*\|_\H^2} \\
& \leq & \frac{\|\w_i - \w_*\|_{\Z_i}^2}{2\eta_i} - \frac{\|\w_{i+1} - \w_*\|_{\Z_i}^2}{2\eta_i} + \frac{\eta G^2}{2}\x_i^{\top}\Z_i\x_i + (\w_i - \w_*)^{\top}\delta_i.
\end{eqnarray*}
Using the fact that $\eta_1 \geq 3/[\theta\beta]$,
\[
\frac{\|\w_i - \w_*\|_{\Z_i}^2}{2\eta_i} = \frac{\|\w_i - \w_*\|_{\M_i}^2}{\eta_1}
\]
and
\[
\frac{\|\w_{i+1} - \w_*\|_{\Z_i}^2}{2\eta_i} = \frac{\|\w_{i+1} - \w_*\|_{\M_i}^2}{\eta_1} = \frac{\|\w_{i+1} - \w_*\|_{\M_i}^2}{\eta_1} - \frac{1}{2\eta_1}\left[(\w_i - \w_*)^{\top}\x_i\right]^2,
\]
we get
\begin{eqnarray*}
\lefteqn{\lb(\w_i) - \lb(\w_*) + \frac{\theta \beta}{3}\|\w_i - \w_*\|_\H^2} \\
& \leq & \frac{\|\w_i - \w_*\|_{\M_i}^2}{2\eta_1} - \frac{\|\w_{i+1} - \w_*\|_{\M_{i+1}}^2}{2\eta_1} + \frac{\eta_1 G^2}{2}\x_i^{\top}\M^{-1}_i\x_i + (\w_i - \w_*)^{\top}\delta_i \\
&      & + \frac{\theta\beta}{3}\left( \left[(\w_i - \w_*)^{\top}\x_i\right]^2- \|\w_i - \w_*\|_\H^2\right).
\end{eqnarray*}
%
\section*{Appendix B. Proof of Lemma~\ref{lemma:trace}}
Since $\M_i = \M_{i-1} + \x_i\x_i^{\top}$, we have
\[
\x_i^{\top} \M_i^{-1} \x_i^{\top} = \tr(\M_i^{-1}(\M_i - \M_{i-1}) = \tr(\I - \M_i^{-1} \M_{i-1}).
\]
Let $\gamma_j \geq 0, j=1,\ldots, \ldots,d$ be the eigenvalues of $\M_i^{-1/2}\M_{i-1}\M_i^{-1/2}$. We have
\begin{eqnarray*}
\tr\left(\I - \M_i^{-1}\M_{i-1}\right) & = & \tr\left(\I - \M_i^{-1/2}\M_{i-1}\M_i^{-1/2}\right) \\
& = & \sum_{j=1}^d (1 - \gamma_j) \\
 &\leq& \sum_{j=1}^d \ln\frac{1}{\gamma_j} \\
&= & - \ln\det\left(\M_i^{-1/2}\M_{i-1}\M_i^{-1/2}\right)  =  \ln\det(\M_i) - \ln\det(\M_{i-1}),
\end{eqnarray*}
which concludes the proof.
%
\subsection*{Appendix C. Proof of Lema~\ref{lem:logdet}}
Define $\V = \M_n - \M_0 = \sum_{i=1}^n \x_i\x_i^{\top}$. Let $\lambda_j, j=1, \ldots, d$ be the eigenvalues of $\V$. It is easy to verify that
\[
\log\det(\M_n) - \log\det(\M_0) = \sum_{j=1}^d \log\left(1 + \frac{\lambda_j}{a}\right)
\]
Since $\sum_{j=1}^d \lambda_j = n$, we have
\[
\log\det(\M_n) - \log\det(\M_0) \leq \max\limits_{\lambda_j \geq 0, \sum_{j=1}^d \lambda_j \leq n} \sum_{j=1}^d \log\left(1 + \frac{\lambda_j}{a}\right).
\]
It is easy to verify that the above optimization takes its optimal at $\lambda_j = n/d, j=1, \ldots, d$.

%
\section*{Appendix D. Proof of Lemma~\ref{lem:Delta2}}
The proof is based on the Bernstein inequality for martingales~(see e.g.,~\cite{boucheron2004concentration}).
\begin{thm} \label{thm:bernstein} (Bernstein's inequality for martingales). Let $X_1, \ldots , X_n$ be a bounded martingale difference sequence with respect to the filtration $\F = (\F_i)_{1\leq i\leq n}$ and with $\|X_i\| \leq K$. Let
\[
S_i = \sum_{j=1}^i X_j
\]
be the associated martingale. Denote the sum of the conditional variances by
\[
    \Sigma_n^2 = \sum_{t=1}^n \E\left[X_t^2|\F_{t-1}\right]
\]
Then for all constants $t$, $\nu > 0$,
\[
\Pr\left[ \max\limits_{i=1, \ldots, n} S_i > t \mbox{ and } \Sigma_n^2 \leq \nu \right] \leq \exp\left(-\frac{t^2}{2(\nu + Kt/3)} \right)
\]
and therefore,
\[
    \Pr\left[ \max\limits_{i=1,\ldots, n} S_i > \sqrt{2\nu t} + \frac{\sqrt{2}}{3}Kt \mbox{ and } \Sigma_n^2 \leq \nu \right] \leq e^{-t}
\]
\end{thm}
Define martingale difference
\[
X_i = \langle \w_i - \w_*, \nabla \lb(\w_i) - \ell'\left(y_i \x^{\top}_i \w_i \right)\x_i \rangle
\]
and martingale $\Lambda = \sum_{i=1}^{n} X_i$. Define the conditional variance $\Sigma_n^2$ as
\[
    \Sigma_n^2 = \sum_{i=1}^{n} \E_{i}\left[X_i^2 \right] \leq G^2 \sum_{i=1}^n \|\w_i - \w_*\|_{\H}^2 = G^2 A
\]
Define
\[
K = \max\limits_{i} |X_i| \leq 2RG .
\]
Since $A \leq 4R^2 n$, we have
\begin{eqnarray*}
\lefteqn{\Pr\left(\Lambda \geq 2G\sqrt{A t} + \sqrt{2}Kt/3\right)} \\
& = & \Pr\left(\Lambda \geq 2G\sqrt{A t} + \sqrt{2}Kt/3, A \leq 4R^2 n\right) \\
& =  & \Pr\left(\Lambda \geq 2G\sqrt{A t} + \sqrt{2}Kt/3, \Sigma_n^2 \leq G^2 A, A \leq 4R^2 n \right) \\
& \leq  & \Pr\left(\Lambda \geq 2G\sqrt{A t} + \sqrt{2}Kt/3, \Sigma_n^2 \leq G^2 A, A \leq \frac{4R^2}{n} \right) \\
&  & + \sum_{i=1}^m \Pr\left(\Lambda \geq 2G\sqrt{At} + \sqrt{2}Kt/3, \Sigma_n^2 \leq G^2 A, \frac{2^{i+1}R^2}{n} < A  \leq \frac{2^{i+2}R^2}{n}  \right) \\
& \leq  & \Pr\left(A \leq \frac{4R^2}{n}\right) + \sum_{i=1}^m \Pr\left(\Lambda \geq 4GR\sqrt{\frac{2^{i}}{n} t} + \sqrt{2}Kt/3, \Sigma_n^2 \leq \frac{4R^2G^2}{n} 2^i\right) \\
& \leq  & \Pr\left(A \leq \frac{1}{n}\right) + me^{-t}
\end{eqnarray*}
where $m = \lceil 2\log_2 n \rceil$. The last step follows the Bernstein inequality for martingales. We complete the proof by setting $t= \ln(m/\delta)$ and using the fact
\[
    2G\sqrt{A} = \frac{6G^2}{\theta\beta} + \frac{\theta\beta}{6}A.
\]

%
\section*{Appendix E. Proof of Lemma~\ref{lemma:bound-3}}
\begin{eqnarray*}
\sum_{i=1}^{n} \lb(\w_i) - \lb(\w_*) \leq \sum_{i=1}^{n} \E_{(\x, y)}\left[\langle \w_i - \w_*, \ell'(y, \x^{\top}\w_i) \x \rangle \right] - \frac{\theta}{2}\sum_{i=1}^{n}\langle \w_i - \w_*, \H(\w_i - \w_*) \rangle
\end{eqnarray*}
Since
\begin{eqnarray*}
\lefteqn{\sum_{i=1}^{n} \E_{(\x, y)}\left[\langle \w_i - \w_*, \ell'(y, \x^{\top}\w_i) \x \rangle \right]} \\
& \leq & \sqrt{n}\sqrt{\sum_{i=1}^{n} \left(\E_{(\x, y)}\left[\langle \w_i - \w_*, \ell'(y, \x^{\top}\w_i) \x \rangle \right] \right)^2} \\
& \leq & G \sqrt{n} \sqrt{\sum_{i=1}^{n} \langle \w_i - \w_*, \H (\w_i - \w_*)\rangle} \\
& = & G\sqrt{n A} \leq 2RG,
\end{eqnarray*}
we obtain the desired inequality as
\[
\sum_{i=1}^{n} \lb(\w_i) - \lb(\w_*) + \frac{\theta}{2}\sum_{i=1}^{n}\langle \w_i - \w_*, \H(\w_i - \w_*) \rangle \leq 2RG.
\]
%
%

\end{document}